\documentclass[letterpaper]{article}
\usepackage{ijcai13}
\usepackage{times}
\usepackage{helvet}
\usepackage{courier}
\usepackage{hyperref}
\usepackage{graphicx}
\usepackage{latexsym}
\usepackage{amsmath,amssymb,amsfonts,amsthm}
\usepackage{graphicx}
\usepackage{hyperref}
\usepackage{multirow}
\usepackage[linesnumbered, ruled, noline, noend]{algorithm2e}
\usepackage{amsfonts}
\usepackage{color}
\definecolor{darkblue}{rgb}{0.1,0.1,0.4}

\newtheorem{lemma}{Lemma}
\newtheorem{proposition}{Proposition}

\newtheorem{definition}{Definition}
\newtheorem{example}{Example}

\newtheorem{property}{Property}

\newcommand{\myOmit}[1]{}

\newcommand{\FOCUS} {\mbox{\sc Focus}}
\newcommand{\FOCUSH} {\mbox{\sc SpringyFocus}}

\newcommand{\FOCUSW} {\mbox{\sc WeightedFocus}}

\newcommand{\SpringyFocus} {\mbox{\sc SpringyFocus}}
\newcommand{\FOCUSWH} {\mbox{\sc WeightedSpringyFocus}}
\newcommand{\AMONG} {\mbox{\sc Among}}
\newcommand{\REGULAR} {\mbox{\sc Regular}}
\newcommand{\len}{\mathit{len}}
\newcommand{\h}{\mathit{h}}
\newcommand{\kv}{\mathit{k}}
\newcommand{\numb}{\mathit{y_c}}
\newcommand{\fc}{\mathit{fc}}
\newcommand{\C}{\mathit{z_c}}
\newcommand{\cU}{\mathit{z_c^U}}

\newcommand{\Start}{\mathit{S}}
\newcommand{\End}{\mathit{E}}

\newcommand{\Taken}{\mathit{T}}
\newcommand{\LastTaken}{\mathit{last_c}}
\newcommand{\Zero}{\mathit{Z}}

{\makeatletter
 \gdef\xxxmark{%
   \expandafter\ifx\csname @mpargs\endcsname\relax 
     \expandafter\ifx\csname @captype\endcsname\relax 
       \marginpar{xxx}
     \else
       xxx 
     \fi
   \else
     xxx 
   \fi}
 \gdef\xxx{\@ifnextchar[\xxx@lab\xxx@nolab}
 \long\gdef\xxx@lab[#1]#2{{\bf [\xxxmark #2 ---{\sc #1}]}}
 \long\gdef\xxx@nolab#1{{\bf [\xxxmark #1]}}
}

\title{Three Generalizations of the FOCUS Constraint}
\author{Nina Narodytska\\ 
NICTA and UNSW \\ 
Sydney, Australia \\ 
ninan@cse.unsw.edu.au
\And Thierry Petit \\
LINA-CNRS \\
Mines-Nantes,  INRIA \\
Nantes, France \\
thierry.petit@mines-nantes.fr \\
 \And Mohamed Siala \\
LAAS-CNRS\\
Univ de Toulouse, INSA\\
Toulouse, France  \\
 msiala@laas.fr 
 \And Toby Walsh \\ 
 NICTA and UNSW \\ 
Sydney, Australia \\ 
toby.walsh@nicta.com.au
 }

\begin{document}
\maketitle
\begin{abstract}
The \FOCUS~constraint expresses the notion that solutions
are concentrated. In practice,
this constraint suffers from the rigidity of its semantics. To tackle this issue, we propose three
generalizations of the \FOCUS\ constraint. 
We provide for each one a complete filtering algorithm as well as
discussing decompositions.
\end{abstract}
\section{Introduction}
Many discrete optimization problems have
constraints on the objective function.
Being able to represent such constraints is fundamental to deal 
with many real world industrial problems.
Constraint programming is a promising approach to
express and filter such constraints. 
In particular, several constraints have been proposed for obtaining well-balanced solutions~\cite{pesreg05,schdevdupreg07,petreg11}.
Recently, the \FOCUS~constraint~\cite{petit12} was introduced to express the opposite notion. 
It captures the concept of concentrating the high values
in a sequence of variables 
to a small number of intervals. 
We recall its definition.
Throughout this paper, $X=[x_0, x_1, \ldots, x_{n-1}]$ is a sequence of variables and $s_{i,j}$ is a sequence of indices of consecutive variables in $X$,
such that $s_{i,j} = [i, i+1,\ldots, j]$, $0 \leq i \leq j <n$.  We let $|E|$ be the size of a collection $E$.
\begin{definition}[\cite{petit12}]\label{def:focus}
Let 
$\numb$ be a variable. Let $\kv$ and $\len$ be two integers,
$1 \leq \len \leq |X|$. An instantiation of $X \cup \{\numb\}$ satisfies \FOCUS($X,\numb,\len,\kv$)
iff there exists a set $S_X$ of \emph{disjoint} sequences of indices $s_{i,j}$ such that three conditions are all satisfied:
\emph{\bf(1)} $|S_X| \leq \numb$
\emph{\bf(2)}  $\forall x_l \in X$, $x_l>\kv \Leftrightarrow \exists s_{i,j} \in S_X$ such that $l \in s_{i,j}$
\emph{\bf(3)} $\forall s_{i,j} \in S_X$, $j-i+1 \leq \len$
\end{definition}

\FOCUS~can be used in various contexts including cumulative scheduling problems where some excesses of capacity can be tolerated to obtain a solution~\cite{petit12}.
In a cumulative scheduling problem, we are scheduling
activities, and each activity consumes a certain amount of some 
resource. The total quantity of the resource available
is limited by a capacity.  Excesses can 
be represented by variables~\cite{declercqal11}.
In practice, excesses might be tolerated 
by, for example, renting a new machine to produce more resource.
Suppose the rental price decreases proportionally to its duration: 
it is cheaper to rent a machine during a single interval than
to make several rentals. On the other hand, rental intervals 
have generally a maximum possible duration.
\FOCUS~can be set to concentrate (non null) excesses in a small number 
of intervals, each of length at most $\len$.

Unfortunately, the usefulness of \FOCUS~is hindered by the rigidity of 
its semantics. For example, 
we might be able to 
rent a machine from Monday to Sunday but not use it on Friday.
It is a pity to miss such a solution with a smaller number of 
rental intervals because \FOCUS~imposes that all the variables within each rental interval take a high value.
Moreover, a solution with one rental interval of two days is
better than a solution with a rental interval of four days. 
Unfortunately, \FOCUS~only considers
the \emph{number} of disjoint sequences,
and does not consider their \emph{length}.

We tackle those issues here by means of three generalizations of \FOCUS. \SpringyFocus~tolerates within each sequence in $s_{i,j} \in S_X$
some values $v \leq \kv$.  To keep the semantics of grouping high values, their number is limited in each $s_{i,j}$ by an
integer argument.  $\FOCUSW$~adds a variable to count the length of sequences, equal to the number
of variables taking a value $v > \kv$. The most generic one,  \FOCUSWH, combines the semantics of \SpringyFocus~and \FOCUSW.
Propagation of constraints like these complementary to an objective function is well-known to be important~\cite{petpod08,schvanreg09}.
We present and experiment with filtering algorithms and decompositions
therefore for each constraint.
\section{Springy FOCUS}\label{sec:focush}
In Definition~\ref{def:focus}, each sequence in $S_X$ contains \emph{exclusively}
values $v>k$. In many practical cases, this property is too strong.
\begin{figure}[!h]
\begin{center}
  \includegraphics[width=0.48\textwidth]{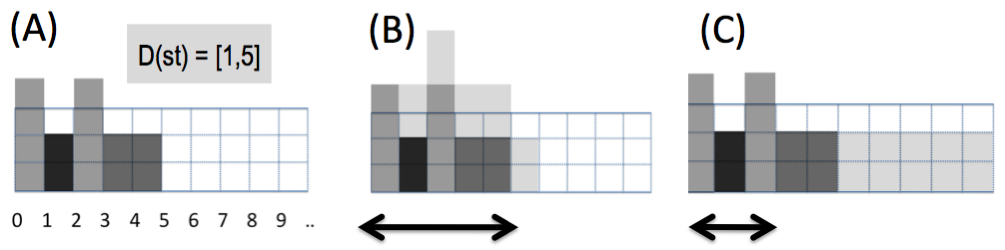}
\end{center}
\vspace{-0.3cm}
\caption{\footnotesize(A) Problem with 4 fixed activities
and one activity of length 5 that can start from time 1 to 5. (B) Solution satisfying
\FOCUS($X, [1,1],5,0$), with a new machine rented for 5 days. (C) Practical solution violating
\FOCUS($X, [1,1],5,0$), with a new machine rented for 3 days but not used on the second day.
}
\vspace{-0.3cm}
  \label{fig:SF}
\end{figure}
Consider one simple instance of the problem in the introduction, in Figure~\ref{fig:SF},
where one variable $x_i \in X$ is defined per point in time $i$ (e.g., one day), to represent excesses of capacity.
Inintialy, 4 activities are fixed and one activity $a$ remains
to be scheduled (drawing A), of duration 5
and that can start from day 1 to day 5.
If \FOCUS($X, \numb=1,5,0$) is imposed then $a$ must start at day $1$ (solution B). We have one $5$ day
rental interval. Assume now that the new machine may not be
used every day. Solution (C) gives one rental of $3$ days instead of $5$.
Furthermore, if $\len=4$ the problem will have no solution using \FOCUS, while this latter solution still exists in practice.
This is paradoxical, as relaxing
the condition 
that sequences in the set $S_X$ of Definition~\ref{def:focus} take only
values $v>k$ deteriorates the concentration power of the constraint. Therefore,
we propose a soft relaxation of \FOCUS, where \emph{at most} $\h$ values less
than $\kv$ are tolerated within each sequence in $S_X$.
\begin{definition} \label{def:SF}
Let $\numb$ be a variable and $\kv$, $\len$, $\h$ be three integers,
$1\leq\len\leq|X|$, $0$ $\leq \h<\len-1$. An instantiation of $X \cup \{\numb\}$ satisfies \SpringyFocus($X,\numb,\len,\h,\kv$) iff there exists
a set $S_X$ of \emph{disjoint} sequences of indices $s_{i,j}$ such that four conditions are all satisfied:
\emph{\bf(1)}
$|S_X| \leq \numb$
\emph{\bf(2)}
$\forall x_l \in X$, $x_l>\kv \Rightarrow \exists s_{i,j} \in S_X$ such that $l \in s_{i,j}$
\emph{\bf(3)}
$\forall s_{i,j} \in S_X$, $j-i+1 \leq \len$, $x_i>k$ and $x_j>k$.
\emph{\bf(4)}
$\forall s_{i,j} \in S_X$, $|\{ l \in s_{i,j}$, $x_l \leq \kv \} |$ $\leq \h$
\end{definition}

Bounds consistency (BC) on \SpringyFocus~is equivalent to domain consistency: any solution can be turned into a solution that only
uses the lower bound $\min(x_l)$ or the upper bound $\max(x_l)$ of the domain $D(x_l)$ of each $x_l \in X$
(this observation was made for \FOCUS~\cite{petit12}). Thus, we propose a BC algorithm.
The first step is to traverse $X$ from $x_0$ to $x_{n-1}$, to compute the minimum possible number of disjoint sequences in $S_X$ (a lower bound for $\numb$),
the \emph{focus cardinality}, denoted $\fc(X)$. We use the same notation for subsequences of $X$. $\fc(X)$ depends
on  $\kv, \len$ and $\h$.

\begin{definition}\label{def:recSF} Given $x_l \in X$, we consider three quantities.
(1) $\underline{p}(x_{l},v_\leq)$ is the focus cardinality of $[x_0, x_1, \ldots, x_l]$, assuming $x_l \leq \kv$, and $\forall s_{i,j} \in S_{[ x_0, x_1, \ldots, x_l ]}, j \neq l$.
(2) $\underline{p_S}(x_{l},v_\leq)$ is the focus cardinality of $[x_0, x_1, \ldots, x_l]$, assuming $x_l \leq \kv$ and $\exists i, s_{i,l} \in S_{[ x_0, x_1, \ldots, x_l ]}$.
(3) $\underline{p}(x_{l},v_>)$  is the focus cardinality of $[x_0, x_1, \ldots, x_l]$ assuming $x_l > \kv$.

Any quantity is equal to $n+1$ if the domain $D(x_l)$ of $x_l$ makes not possible the considered assumption.
\end{definition}
\begin{property}\label{prop:fc}
$\underline{p_S}(x_{0},v_\leq)=\underline{p_S}(x_{n-1},v_\leq)=n+1$, and $\fc(X)=\min(\underline{p}(x_{n-1},v_\leq), \underline{p}(x_{n-1},v_>))$.
\end{property}
\begin{proof} By construction from Definitions~\ref{def:SF} and~\ref{def:recSF}.
\end{proof}

To compute the quantities of Definition~\ref{def:recSF} for $x_l$$\in$$X$ we use
$\underline{\mathit{plen}}(x_{l})$, the minimum length of a sequence in $S_{[ x_0, x_1, \ldots, x_l ]}$ containing
$x_l$ among instantiations of $[ x_0, x_1, \ldots, x_l ]$ where the number of sequences is $\fc([x_0,x_1,\ldots, x_l])$.
$\underline{\mathit{plen}}(x_{l})$$=$$0$ if $\forall s_{i,j} \in S_{[ x_0, x_1, \ldots, x_l ]}, j \neq l$.
$\underline{\mathit{card}}(x_{l})$ is the minimum number of values $v \leq k$ in the current sequence in $S_{[ x_0, x_1, \ldots, x_l ]}$, equal to $0$ if $\forall s_{i,j} \in S_{[ x_0, x_1, \ldots, x_l ]}, j \neq l$.
$\underline{\mathit{card}}(x_{l})$ assumes that $x_l >k$. It has to be decreased it by one if $x_l \leq k$.
For sake of space, proofs of next 
lemmas are given in Appendix. 

\begin{lemma}[initialization] \label{lem:initialization}
$\underline{p}(x_{0},v_\leq)= 0$ if $\min(x_0) \leq k$, and $n+1$ otherwise;
$\underline{p_{S}}(x_{0},v_\leq)=n+1$;
$\underline{p}(x_{0},v_>)$ $=$ $1$ if $\max(x_0)>k$ and $n+1$ otherwise;
$\underline{\mathit{plen}}(x_{0})$ $=$ $1$ if $\max(x_0)>k$ and $0$ otherwise;
$\underline{\mathit{card}}(x_{0})$ $=$ $0$.
\end{lemma}

\begin{lemma}[$\underline{p}(x_{l},v_\leq)$]\label{lem:pleq}
If $\min(x_l) \leq k$ then
$\underline{p}(x_{l},v_\leq)$ $=$ $\min(\underline{p}(x_{l-1},v_\leq), \underline{p}(x_{l-1},v_>))$, else $\underline{p}(x_{l},v_\leq)$ $=$ $n+1$.
\end{lemma}

\begin{lemma}[$\underline{p_S}(x_{l},v_\leq)$]\label{lem:pleqs}
If $\min(x_i)$$>$$k$, $\underline{p_S}(x_{i},v_\leq)$$=$$n+1$. \\
Otherwise,
if $\underline{\mathit{plen}}(x_{i-1})$ $\in$ $\{0, \len-1, \len \}$ $\vee$ $\underline{\mathit{card}}(x_{i-1})$ $=$ $h$ then $\underline{p_S}(x_{i},v_\leq)$ $=$ $n+1$,
else $\underline{p_S}(x_{i},v_\leq)$ $=$ $\min(\underline{p_S}(x_{i-1},v_\leq), \underline{p}(x_{i-1},v_>))$.
\end{lemma}

\begin{lemma}[$\underline{p}(x_{l},v_>)$]\label{lem:pgt}
If $\max(x_l) \leq k$ then
$\underline{p}(x_{l},v_>)$$=$$n+1$. \\
Otherwise,
If $\underline{\mathit{plen}}(x_{l-1})$ $\in$ $\{0, \len\}$, $\underline{p}(x_{l},v_>)$ $=$ $\min(\underline{p}(x_{l-1},v_>)+1, \underline{p}(x_{l-1},v_\leq) +1)$,
else $\underline{p}(x_{l},v_>)$ $=$ $\min(\underline{p}(x_{l-1},v_>), \underline{p_S}(x_{l-1},v_\leq), \underline{p}(x_{l-1},v_\leq) +1)$.
\end{lemma}

\begin{proposition}[$\underline{\mathit{plen}}(x_{l})$]\label{prop:plen}
(by construction) If $\min$ $(\underline{p_S}(x_{l-1},v_\leq),$$\underline{p}(x_{l-1},v_>))$$<$$\underline{p}(x_{l-1},$$v_\leq)$$+$$1$$\wedge$$\underline{\mathit{plen}}$$(x_{l-1})$$<\len$ then
$\underline{\mathit{plen}}(x_{l}) = \underline{\mathit{plen}}(x_{l-1})  +1$.
Otherwise,  if $\underline{p}(x_{l},v_>)) <  n+1$ then $\underline{\mathit{plen}}(x_{l}) = 1$,
else $\underline{\mathit{plen}}(x_{l}) = 0$.
\end{proposition}

\begin{proposition}[$\underline{\mathit{card}}(x_{l})$]\label{prop:card} (by construction)
If $\underline{\mathit{plen}}(x_{l})=1$ then $\underline{\mathit{card}}(x_{l})$ $=$ $0$.
Otherwise,
if $\underline{p}(x_{l},v_>)=n+1$ then $\underline{\mathit{card}}(x_{l})$ $=$ $\underline{\mathit{card}}(x_{l-1}) + 1$,
else $\underline{\mathit{card}}(x_{l})$ $=$ $\underline{\mathit{card}}(x_{l-1})$.
\end{proposition}
{
\begin{algorithm}[!h]\label{alg:mincard}
\scriptsize
\caption{\scriptsize{\sc MinCards}($X, \len, \kv, \h$): Integer matrix} \label{alg:minCard}
$\mathit{pre}$ $:=$ new Integer$[|X|][4][]$ \;
\For{$l \in 0..n-1$}{
  $\mathit{pre}[l][0]$ $:=$ new Integer$[2]$\;
 \lFor{$j \in 1..3$}{
   $\mathit{pre}[l][j]$ $:=$ new Integer$[1]$\;
   }
   }
Initialization Lemma~\ref{lem:initialization},\;
\lFor{$l \in 1..n-1$}{
	Lemmas~\ref{lem:pleq}, ~\ref{lem:pleqs}, ~\ref{lem:pgt} and Propositions ~\ref{prop:plen} and ~\ref{prop:card}.\;
}
return $\mathit{pre}$\;
\end{algorithm}
}

Algorithm~\ref{alg:mincard} implements the lemmas with
$pre[l][0][0]=\underline{p}(x_{l},v_\leq)$,
$pre[l][0][1]=\underline{p_{S}}(x_{l},v_\leq)$,
$pre[l][1]=\underline{p}(x_{l},v_>)$,
$pre[l][2]=\underline{\mathit{plen}}(x_{l})$,
$pre[l][3]=\underline{\mathit{card}}(x_{l})$.

The principle of Algorithm~\ref{alg:gac} is the following. First, $lb=fc(X)$ is computed with $x_{n-1}$.
We execute Algorithm~\ref{alg:mincard} from $x_0$
to $x_{n-1}$ and conversely (arrays $\mathit{pre}$ and $\mathit{suf}$). We thus have for each quantity two values for each variable $x_l$.
To aggregate them, we implement regret mechanisms directly derived from Propositions~\ref{prop:card} and~\ref{prop:plen}, according to
the parameters $\len$ and $\h$.
{
\begin{algorithm}[!t]
\scriptsize
\caption{\scriptsize{\sc Filtering}($X, \numb, \len, \kv, \h$): Set of variables} \label{alg:gac}
$\mathit{pre}$ $:=$ {\sc MinCards}$(X,\len,k,h)$ \;
Integer $lb$ $:=$ $\min(\mathit{pre}[n-1][0][0], \mathit{pre}[n-1][1])$\;
\lIf{$\min(\numb)<lb$}{
  $D(\numb)$ $:=$ $D(\numb) \setminus [\min(\numb), lb[$\;
}
\If{$\min(\numb) = \max(\numb)$}{
  $\mathit{suf}$ $:=$  {\sc MinCards}$([x_{n-1}, x_{n-2}, \ldots, x_0],\len,k,h)$ \;
 \For{$l \in 0..n-1$}{
   	 \If{ $\mathit{pre}[l][0][0] + \mathit{suf}[n-1-l][0][0] > \max(\numb)$} {
		 Integer $\mathit{regret}$ $:=$ $0$;
		 Integer $add$ $:=$ $0$\;
		 \lIf{ $\mathit{pre}[l][1] \leq \mathit{pre}[l][0][1]$}{
		 	$add$ $:=$ $add$ $+$ $1$\;
		 }
		 \lIf{ $\mathit{suf}[n-1-l][1] \leq \mathit{suf}[n-1-l][0][1]$}{
		 	$add$$:=$$add$$+$$1$\;
		 }
    	  	\lIf{$\mathit{pre}[l][2] + \mathit{suf}[n-1-l][2] -1 \leq \len$ $\wedge$ $\mathit{pre}[l][3] + \mathit{suf}[n-1-l][3] + add -1 \leq h$}{
    			$\mathit{regret}$ $:=$ $1$\;
    		}
		\lIf{$\mathit{pre}[l][0][1] + \mathit{suf}[n-1-l][0][1] - regret > \max(\numb)$}{	
  			$D(x_i)$ $:=$ $D(x_i) \setminus$ $[\min(x_i), k]$\;
		}
		}
     		Integer $\mathit{regret}$ $:=$ $0$\;
       		\lIf{$\mathit{pre}[l][2] + \mathit{suf}[n-1-l][2] -1 \leq \len$ $\wedge$ $\mathit{pre}[l][3] + \mathit{suf}[n-1-l][3] -1$ 
		       $\leq h$}{
    			$\mathit{regret}$ $:=$ $1$\;
    		}
    		\If{$\mathit{pre}[l][1] + \mathit{suf}[n-1-l][1] - regret > \max(\numb)$}{
       			$D(x_i)$ $:=$ $D(x_i) \setminus$  $]k, \max(x_i)]$\;
    		}
	
}
}
return $X$ $\cup$ $\{ \numb \}$\;
\end{algorithm}
}
Line 4 is optional but it avoids some work when the
variable $\numb$ is fixed, thanks to the same property as \FOCUS~(see \cite{petit12}).  Algorithm~\ref{alg:gac} performs a constant number of
traversals of the set $X$. Its time complexity is $O(n)$, which is optimal.

\section{Weighted FOCUS}\label{sec:focusw}
We present $\FOCUSW$, that extends $\FOCUS$ with a variable $\C$
limiting the the sum of lengths of all the sequences in $S_X$, i.e., the number of variables \emph{covered} by a
sequence in $S_X$. It distinguishes between solutions that are equivalent
 with respect to the number of sequences in $S_X$ but not with respect to their length, as Figure~\ref{fig:WF} shows.
 

\begin{definition} \label{def:focus_basic_cost}
Let
$\numb$ and $\C$ be two integer variables and $\kv$, $\len$ be two integers,
such that $1 \leq \len \leq |X|$. An instantiation of $X \cup \{\numb\} \cup \{\C\}$ satisfies \FOCUSW($X,\numb,\len,\kv, \C$)
iff there exists a set $S_X$ of \emph{disjoint} sequences of indices $s_{i,j}$ such that four conditions are all satisfied:
\emph{\bf(1)}
$|S_X| \leq \numb$ \label{cond:def:card}
\emph{\bf(2)}
$\forall x_l \in X$, $x_l>\kv \Leftrightarrow \exists s_{i,j} \in S_X$ such that $l \in s_{i,j}$ \label{cond:def:cover}
\emph{\bf(3)}
$\forall s_{i,j} \in S_X$, $j-i+1 \leq \len$  \label{cond:def:len}
\emph{\bf(4)}
$ \sum_{s_{i,j} \in S_X} |s_{i,j}| \leq \C$. \label{cond:def:cost}
\end{definition}
\begin{definition}[\cite{petit12}] Given an integer $\kv$, a variable $x_l \in X$ is:
Penalizing, $(P_\kv)$, iff $min(x_l) > k$.
Neutral, $(N_\kv)$, iff $max(x_l) \leq  k$.
Undetermined, $(U_\kv)$, otherwise.
We say $x_ l \in P_k$ iff $x_l$ is labeled $P_k$, and similarly for $U_k$ and $N_k$.
\end{definition}

\begin{figure}[!h]
\begin{center}
  \includegraphics[width=0.48\textwidth]{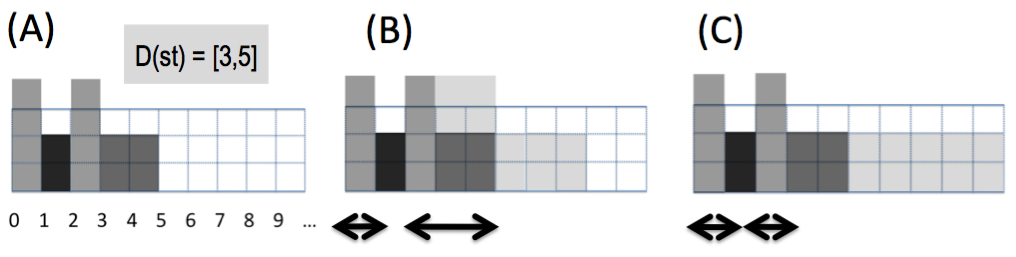}
\end{center}
\vspace{-0.3cm}
\caption{\footnotesize(A) Problem with 4 fixed activities
and one activity of length 5 that can start from time 3 to 5. We assume $D(\numb)=\{2\}$, $\len=3$ and $\kv=0$.
(B) Solution satisfying
\FOCUSW~with $\C=4$. (C) Solution satisfying
\FOCUSW~with $\C=2$.
}
\vspace{-0.2cm}
  \label{fig:WF}
\end{figure}


\paragraph{Dynamic Programming (DP) Principle}
Given a partial instantiation $I_X$ of $X$
and a set of sequences $S_X$ that covers all penalizing variables in $I_X$, we
consider two terms: the number of variables in
$P_k$ and the number of \emph{undetermined }variables, in $U_k$, covered by $S_X$.
We want to find a set $S_X$ that minimizes the second term.
Given a sequence of variables $s_{i,j}$, the cost $cst(s_{i,j})$ is defined as
$cst(s_{i,j}) = \{p|x_p \in U_k, x_p \in s_{i,j}\}$.
We denote cost of $S_X$, $cst(S_X)$, the sum $cst(S_X) = \sum_{s_{i,j} \in S_X} cst(s_{i,j})$.
Given $I_X$ we consider $|P_k| = |\{x_i \in P_k\}|$. We have:
$ \sum_{s_{i,j} \in S} |s_{i,j}| = \sum_{s_{i,j} \in S} cst(s_{i,j}) + |P_k|$.



We start with explaining the main difficulty in building a propagator for $\FOCUSW$.
The constraint has two optimization variables in its scope and we might  not have a solution
that optimizes both variables simultaneously.
\begin{example}\label{exm:focusw_twocostvars}
Consider the set $X=[x_0,x_1,\ldots,x_5]$ with domains $[1,\{0,1\},1,1,\{0,1\},1]$
and $\FOCUSW(X,[2,3],3,0,[0,6])$, solution
$S_X = \{s_{0,2}, s_{3,5}\}$, $\C = 6$,  minimizes $\numb= 2$, while solution
$S_X = \{s_{0,1}, s_{2,3}, s_{5,5} \}$, $\numb = 3$, minimizes $\C = 4$.
\end{example}

Example~\ref{exm:focusw_twocostvars} suggests that we need to fix one of the two optimization variables
and only optimize the other one.
Our algorithm is based on a dynamic program~\cite{daspapvaz06}. For each prefix of variables
$[x_0,x_1,\ldots,x_j]$ and \emph{given} a cost value $c$, it computes a cover of focus cardinality, denoted $S_{c,j}$,
which covers all penalized variables in $[x_0, x_1,\ldots,x_j]$ and has cost \emph{exactly}
$c$. If $S_{c,j}$ does not exist we assume that $S_{c,j} = \infty$.
$S_{c,j}$ is not unique as Example~\ref{exm:four}
demonstrates.
\begin{example}\label{exm:four}
Consider  $X=[x_0,x_1,\ldots,x_{7}]$ and  $\FOCUSW(X, [2,2], 5, 0, [7,7])$,
with $D(x_i) = \{1\}$,
$i \in  I, I = \{0,2,3,5,7\}$ and
$D(x_i)= \{0, 1\}$, $i \in \{0,1, \ldots 7\} \setminus I$.
Consider the subsequence of variables $[x_0,\ldots,x_5]$
and $S_{1,5}$. There are several sets of minimum cardinality
that cover all penalized variables in the prefix $[x_0,\ldots,x_5]$ and has cost
$2$, e.g. $S_{1,5}^1 = \{s_{0,2}, s_{3,5}\}$ or $S_{1,5}^2 = \{s_{0,4}, s_{5,5}\}$.
Assume we sort sequences by their starting points in each set.
We note that the second set is better if we want to extend the last
sequence in this set as the length of the last sequence $s_{5,5}$ is
shorter compared to the length of the last sequence in $S_{1,5}^1$, which is $s_{3,5}$.
\end{example}

Example~\ref{exm:four} suggests that we need to put additional conditions on $S_{c,j}$
to take into account that some sets are better than others. We
can safely assume that none of the sequences in $S_{c,j}$ starts at undetermined
variables as we can always set it to zero.
Hence, we introduce a notion of an ordering between sets $S_{c,j}$ and define
conditions that this set has to satisfy.

\emph{Ordering of sequences in $S_{c,j}$.} We introduce an order over sequences in $S_{c,j}$. Given a set of sequences in $S_{c,j}$ we sort them by their starting points.
We denote $last(S_{c,j})$ the last sequence in $S_{c,j}$ in this order.
If $x_j \in  last(S_{c,j})$ then $|last(S_{c,j})|$ is, naturally,
the length
of $last(S_{c,j})$, otherwise $|last(S_{c,j})| = \infty$.

\emph{Ordering of sets $S_{c,j}$, $c \in [0,\max(\C)]$, $j \in \{0,1,\ldots,n-1\}$.}
We define a comparison operation between two sets $S_{c,j}$ and $S_{c',j'}$.
$S_{c,j} \leq S_{c',j'}$ iff
$|S_{c,j}| < |S_{c',j'}|$ or $|S_{c,j}| = |S_{c',j'}|$ and $last(S_{c,j}) \leq last(S_{c',j'})$.
Note that we do not take account of cost in the comparison as the current definition is sufficient for us.
Using this operation, we can compare all sets $S_{c,j}$ and $S'_{c,j}$ of the same cost for a prefix $[x_0,\ldots,x_j]$.
We say that $S_{c,j}$ is optimal iff satisfies the following 4 conditions.

\begin{proposition}[Conditions on $S_{c,j}$]~
\begin{enumerate}
  \item $S_{c,j}$ covers all $P_k$ variables in  $[x_0,x_1, \ldots,x_j]$, \label{cond:cover}
  \item $cst(S_{c,j})=c$, \label{cond:cost}
  \item $\forall s_{h,g} \in S_{c,j}, x_h \notin U_k$, \label{cond:start}
  \item $S_{c,j}$ is the first set in the order among all sets that satisfy conditions~\ref{cond:cover}--\ref{cond:start}\label{cond:last}.
\end{enumerate}
\end{proposition}

As can be seen from definitions above, given a
subsequence of variables $x_0,\ldots, x_j$,  $S_{c,j}$ is not unique and might not exist.
However, if $|S_{c,j}| = |S_{c',j'}|$, $c = c'$ and $j = j'$,  then
$last(S_{c,j}) = last(S_{c',j'})$.

\begin{example}\label{exm:four_con}
Consider  $\FOCUSW$  from Example~\ref{exm:four}.
Consider the subsequence 
$[x_0,x_1]$.
$S_{0,1} = \{s_{0,0}\}$,
$S_{1,1} = \{s_{0,1}\}$.
Note that
$S_{2,1}$ does not exist.
Consider the subsequence 
$[x_0,\ldots,x_5]$. We have
$S_{0,5} = \{s_{0,0}, s_{2,3}, s_{5,5}\}$,
$S_{1,5} = \{s_{0,4}, s_{5,5}\}$ and
$S_{2,5} = \{s_{0,3}, s_{5,5}\}$.
By definition, $last(S_{0,5}) = s_{5,5}$, $last(S_{1,5}) = s_{5,5}$
and $last(S_{2,5}) = s_{5,5}$.
Consider the set $S_{1,5}$. Note that there exists another set $S'_{1,5} = \{s_{0,0}, s_{2,5}\}$
that satisfies conditions~\ref{cond:cover}--\ref{cond:start}.
Hence, it has the same cardinality as $S_{1,5}$
and the same cost. However, $S_{1,5} < S'_{1,5}$ as
$|last(S_{1,5})| = 1 < |last(S'_{1,5})| = 3$.
\end{example}

\paragraph{Bounds disentailment}
Each cell in the dynamic programming table $f_{c,j}$, $c \in [0, \cU]$, $j \in \{0,1,\ldots,n-1\}$, where $\cU = max(\C) - |P_k|$,
is a pair of values $q_{c,j}$ and $l_{c,j}$,  $f_{c,j} =\{q_{c,j}, l_{c,j}\}$, stores information about $S_{c,j}$.
Namely, $q_{c,j} = |S_{c,j}|$,  $l_{c,j} = |last(S_{c,j})|$ if
$last(S_{c,j}) \neq \infty$ and $\infty$ otherwise.
We say that $f_{c,j}/q_{c,j}/l_{c,j}$
is a dummy (takes a dummy value) iff $f_{c,j}=\{\infty,\infty\}/q_{c,j}  = \infty/l_{c,j}  = \infty$.
If $y_1 =\infty$  and  $ y_2 = \infty$ then we assume that they are equal.
We introduce a dummy variable $x_{-1}$, $D(x_{-1}) = \{0\}$
and a row $f_{-1,j}$, $j=-1,\ldots,n-1$ to keep uniform notations.

{
\begin{algorithm}
\scriptsize
\label{a:disent}
  \caption{Weighted $\FOCUS$($x_0, \ldots, x_{n-1}$)}
  \For{$c \in -1..\cU$}{ \label{a:for_loop_dummy}
      \For{$j \in -1..n-1$}{
         $f_{c,j} \gets \{\infty,\infty\}$\;
      }
  }
  $f_{0,-1} \gets \{0,0\}$
  \;
  \label{a:first_col}
  \For{$j \in 0..n-1$}{ \label{a:for_loop}
      \For{$c \in 0..j$}{
        \If(\tcc*[f]{penalizing}){$ x_j \in P_k$}{
            \uIf
            {$(l_{c,j-1} \in [1,len)) \vee (q_{c,j-1} = \infty)$}{\label{a:cond_P_k}
                 $f_{c,j} \gets  \{q_{c,j-1}, l_{c,j-1} + 1\}$\; \label{a:cond_1}
            }
              \lElse
              {
                 $f_{c,j} \gets   \{q_{c,j-1}+1, 1\}$\; \label{a:cond_2}
              }
        }				
        \If(\tcc*[f]{undetermined }){$ x_j \in U_k$} {
            \lIf{$(l_{c-1,j-1} \in [1,len) \wedge q_{c-1,j-1} = q_{c,j-1}) \vee (q_{c,j-1} = \infty)$}{\label{a:cond_U_k}
                  $f_{c,j} \gets  \{q_{c-1,j-1}, l_{c-1,j-1} + 1\}$
                  \label{a:cond_3}
            }\lElse{
                   $f_{c,j} \gets   \{q_{c,j-1}, \infty\} $   
                   \label{a:cond_4}
            }
        }

        \If(\tcc*[f]{ neutral }){$ x_j \in N_k$}{
           $f_{c,j} \gets  \{q_{c,j-1}, \infty\}$
           \label{a:cond_5}
        }

}
  }
 return $f$\;
\end{algorithm}
}

Algorithm~\ref{a:disent} gives pseudocode for the propagator.
The intuition behind the algorithm is as follows. We emphasize again that by cost 
we mean the number of covered variables in $U_k$.
\begin{figure}
\begin{center}
  \includegraphics[width=0.45\textwidth]{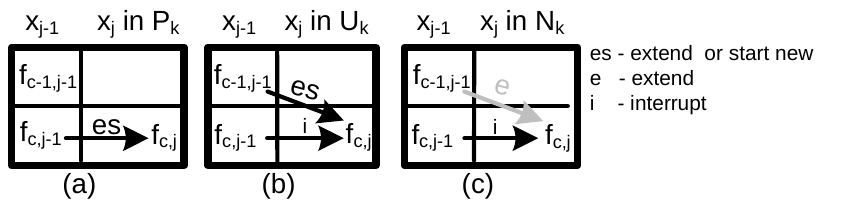}
  \vspace{-0,3cm}
\end{center}
\caption{Representation of one step of Algorithm~\ref{a:disent}.}
\vspace{-0,3cm}
  \label{fig:focusw_scheme}
\end{figure}

 If $x_j \in P_k$ then we do {not} increase the cost of  $S_{c,j}$ compared to $S_{c,j-1}$ as the cost only depends on $x_j \in U_k$. Hence, the best move for us is to extend $last(S_{c,j-1})$ or  start a new sequence if it is possible. This is encoded in lines~\ref{a:cond_1} and~\ref{a:cond_2} of the algorithm. Figure~\ref{fig:focusw_scheme}(a) gives a schematic representation of these arguments.

 If $x_j \in U_k$ then we have two options. We can obtain $S_{c,j}$ from $S_{c-1,j-1}$ by increasing
$cst(S_{c-1,j-1})$  by one. This means that $x_i$ will be covered by $last(S_{c,j})$. Alternatively, from  $S_{c,j-1}$ by interrupting $last(S_{c,j-1})$. 
This is encoded in line~\ref{a:cond_3} 
of the algorithm (Figure~\ref{fig:focusw_scheme}(b)).

 If $x_j \in N_k$ then we do {not} increase the cost of  $S_{c,j}$ compared to $S_{c,j-1}$.
 Moreover, we must interrupt $last(S_{c,j-1})$, line~\ref{a:cond_5} (Figure~\ref{fig:focusw_scheme}(c), ignore the gray arc).

First we prove a property of the dynamic programming table.
We define a comparison operation between $f_{c,j}$ and $f_{c',j'}$ induced
by a comparison operation between $S_{c,j}$ and $S_{c',j'}$:
$f_{c,j} \leq f_{c',j'}$ if $(q_{c,j} < q_{c',j'})$ or  ($q_{c,j} = q_{c',j'}$ and
$l_{c,j} \leq l_{c',j'}$). In other words, as in a comparison operation between sets,
we compare by the cardinality of sequences, $|S_{c,j}|$ and $|S_{c',j'}|$, and,
then by the length of the last sequence in each set, $last(S_{c,j})$
and $last(S_{c',j'})$. 
We omit proofs of the next two lemmas due to space limitations (see Appendix).
\begin{lemma}\label{l:dp_prop_mon}
Consider  $\FOCUSW(X, \numb,\len,\kv, \C)$.
Let $f$ be dynamic programming table returned by Algorithm~\ref{a:disent}.
Non-dummy elements $f_{c,j}$ are monotonically nonincreasing in each column,
so that $f_{c',j} \leq f_{c,j}$, $0 \leq c < c' \leq \cU$, $j=[0,\ldots,n-1]$.
\end{lemma}

\begin{lemma}\label{l:disent}
Consider   $\FOCUSW(X, \numb,\len,\kv, \C)$.
The dynamic  programming table  $f_{c,j} =\{q_{c,j},l_{c,j}\}$ $c \in [0, \cU]$, $j=0,\ldots,n-1$,
is correct in the sense that if $f_{c,j}$ exists and it is non-dummy then a
corresponding set of sequences $S_{c,j}$ exists and
satisfies conditions~\ref{cond:cover}--\ref{cond:last}.
The time complexity of Algorithm~\ref{a:disent} is $O(n\max(\C))$.
\end{lemma}
%

\begin{example}
Table~\ref{table:exm:four} shows an execution
of Algorithm~\ref{a:disent} on $\FOCUSW$ from Example~\ref{exm:four}.
Note that $|P_0| = 5$. Hence, $\cU = max(\C) - |P_0| = 2$.
As can be seen from the table, the constraint has a solution as there exists
a set $S_{2,7}= \{s_{0,3}, s_{5,7}\}$ such that $|S_{2,7}| = 2$.
\end{example}

\begin{table}
\centering
\scriptsize{
\begin{tabular}{|@{}c@{}|@{}c@{}|@{}c@{}|@{}c|@{}c|@{}c|@{}c|@{}c|@{}c|}
\hline
& $D(x_{0})$& $D(x_{1})$& $D(x_{2})$& $D(x_{3})$& $D(x_{4})$& $D(x_{5})$& $D(x_{6})$& $D(x_{7})$\\ 
c & $[1,1]$ & $[0,1]$ & $[1,1]$ & $[1,1]$ & $[0,1]$ & $[1,1]$ & $[0,1]$ & $[1,1]$ \\
\hline
\hline
$0$ & $\{1,1\}$ & $\{1,\infty\}$ & $\{2,1\}$ & $\{2,2\}$ & $\{2,\infty\}$ & $\{3,1\}$ & $\{3,\infty\}$ & $\{4,1\}$ \\
$1$ & & $\{1,2\}$ & $\{1,3\}$ & $\{1,4\}$ & $\{1,\infty\}$ & $\{2,1\}$ & $\{2,\infty\}$& $\{3,1\}$ \\ 
$\cU = 2$ & & & & & $\{1,5\}$ & $\{2,1\}$ & $\{2,2\}$ & $\{2,3\}$ \\

\hline
\end{tabular}
\vspace{-0,1cm}
\caption{An execution of Algorithm~\ref{a:disent} on $\FOCUSW$ from Example~\ref{exm:four}. Dummy values $f_{c,j}$ are removed. \label{table:exm:four}}
\vspace{-0,3cm}
}
\end{table}
\paragraph{Bounds consistency}

To enforce BC on variables $x$, we compute an
additional DP table $b$, $b_{c,j}$, $c \in [0, \cU]$, $j\in[-1,n-1]$
on the reverse sequence of variables $x$.
\begin{lemma}\label{l:focusw_bc}
Consider  $\FOCUSW(X, \numb, \len,\kv, \C)$.
Bounds consistency can be enforced in $O(n\max(\C))$ time.
\end{lemma}
\begin{proof} (Sketch)
We build
dynamic programming tables $f$ and $b$.
We will show that to check if $x_i=v$ has a support it is sufficient to examine $O(\cU)$ pairs of values
$f_{c_1,i-1}$ and $b_{c_2,n-i-2}$, $c_1,c_2 \in [0, \cU]$ which are neighbor columns to the $i$th column.
It is easy to show that if we consider all possible pairs of elements in
$f_{c_1,i-1}$ and $b_{c_2,n-i-2}$ then we determine if there exists a support for $x_i=v$.
There are $O(\cU \times \cU)$ such pairs. The main part of the proof shows that it sufficient to consider
$O(\cU)$ such pairs. In particular, to check a support for a variable-value pair $x_i=v$, $v > k$,
for each $f_{c_1,i-1}$ it is sufficient
to consider only one element $b_{c_2,n-i-2}$ such that
$b_{c_2,n-i-2}$ is non-dummy  and $c_2$ is the maximum value
that satisfies inequality $c_1+c_2 + 1 \leq \cU$.
To check a support for a variable-value pair $x_i=v$, $v \leq k$,
for each $f_{c_1,i-1}$ it is sufficient
to consider only one element $b_{c_2,n-i-2}$ such that
$b_{c_2,n-i-2}$ is non-dummy  and $c_2$ is the maximum value
that satisfies inequality $c_1+c_2 \leq \cU$.
\end{proof}

We observe a useful property of the constraint.
If there exists $f_{c,n-1}$ such that $c < max(\C)$ and $q_{c,n-1} < max(y_c)$
then the constraint is BC. This follows from the observation
that given a solution of the constraint $S_X$, changing a variable value
can increase $cst(S_X)$ and $|S_X|$  by at most one.


%

Alternatively we can decompose \FOCUSW~using
$O(n)$ additional variables and constraints.
\begin{proposition}\label{prop:decomp_focusw}
Given \FOCUS($X,\numb,\len,\kv$), let $\C$ be a variable and
$B$$=$$[b_0, b_{1}, \ldots, b_{n-1}]$ be a set of variables such that $\forall$$b_l$$\in$$B, D(b_l)$$=$$\{0,1\}$.
\FOCUSW($X,\numb,\len,\kv, \C$) $\Leftrightarrow$
{
\FOCUS($X,\numb,\len,\kv$)
$\wedge$ $[\forall l$, $0\leq l <n$,
$[(x_l \leq \kv) \wedge (b_l=0)] \vee [(x_l > \kv) \wedge (b_l=1)]]$
$\wedge$ $\sum_{l \in \{0,1, \ldots, n-1\}} b_l \leq \C$.
}
\end{proposition}

Enforcing BC on each constraint of the decomposition is
weaker than BC on \FOCUSW. Given
$x_l \in X$, a value may have a unique support for $\FOCUS$ which
violates  $\sum_{l \in \{0,1, \ldots, n-1\}} b_l \leq \C$, and conversely.
Consider $n$$=$$5$, $x_0$$=$$x_2$$=$$1$, $x_3$$=$$0$, and
$D(x_1)$$=$$D(x_4)$$=$$\{0,1\}$, $\numb$$=$$2$, $\C$$=$$3$, $\kv$$=$$0$ and $\len$$=$$3$.
Value $1$ for $x_4$ corresponds to this case.

\section{Weighted Springy FOCUS}\label{sec:focuswh}
We consider a further generalization of the $\FOCUS$
constraint that combines $\FOCUSH$ and $\FOCUSW$. 
We prove that we can propagate this constraint in $O(n\max(\C))$ time,
which is same as enforcing BC on $\FOCUSW$.


\begin{definition} \label{def:focuswh}
Let
$\numb$ and $\C$ be two variables and $\kv$, $\len$, $\h$ be three integers,
such that $1 \leq \len \leq |X|$ and $0<\h<\len-1$. An instantiation of $X \cup \{\numb\} \cup {\C}$ satisfies \FOCUSWH($X,\numb, \len, \h, \kv, \C$)
iff there exists a set $S_X$ of \emph{disjoint} sequences of indices $s_{i,j}$ such that five conditions are all satisfied:
\emph{\bf (1)}
$|S_X| \leq \numb$
\emph{\bf (2)}
$\forall x_l \in X$, $x_l>\kv \Rightarrow \exists s_{i,j} \in S_X$ such that $l \in s_{i,j}$
\emph{\bf (3)}
$\forall s_{i,j} \in S_X$, $|\{ l \in s_{i,j}$, $x_l \leq \kv \} |$ $\leq \h$
\emph{\bf (4)}
$\forall s_{i,j} \in S_X$, $j-i+1 \leq \len$,   $x_i>k$ and $x_j>k$.
\emph{\bf (5)}
$ \sum_{s_{i,j} \in S_X} |s_{i,j}| \leq \C$.
\end{definition}


We can again partition cost of $S$ into two terms.
$ \sum_{s_{i,j} \in S} |s_{i,j}| = \sum_{s_{i,j} \in S} cst(s_{i,j}) + |P_k|$.
However, $cst(s_{i,j})$ is the number of \emph{undetermined} and \emph{neutral} variables covered $s_{i,j}$,
$cst(s_{i,j}) = \{p|x_p \in U_k \cup N_k, x_p \in s_{i,j}\}$ as we allow to cover up to $h$
\emph{neutral} variables.

The propagator is again based on a dynamic program that for each prefix of variables
$[x_0, x_1,\ldots,x_j]$ and given cost $c$ computes a cover $S_{c,j}$ of minimum cardinality
that covers all penalized variables in the prefix $[x_0, x_1, \ldots,x_j]$ and has cost \emph{exactly}
$c$. We face the same problem of how to compare two sets $S_{c,j}^1$ and $S_{c,j}^2$
of minimum cardinality.
The issue here is how to compare $last(S_{c,j}^1)$ and $last(S_{c,j}^2)$  if they cover
a different number of neutral variables.
Luckily, we can avoid this problem due to the following monotonicity property.
If $last(S_{c,j}^1)$ and $last(S_{c,j}^2)$ are not equal to infinity then they both
end at the same position $j$. Hence, if $last(S_{c,j}^1) \leq last(S_{c,j}^2)$ then
the number of neutral variables covered by $last(S_{c,j}^1)$ is no
larger than the number of neutral variables covered by $last(S_{c,j}^2)$.
Therefore, we can define order on sets $S_{c,j}$ as we did
in Section~\ref{sec:focusw} for \FOCUSW.


Our bounds disentailment detection algorithm for $\FOCUSWH$ mimics Algorithm~\ref{a:disent}.
We omit the pseudocode due to space limitations but
highlight two not-trivial differences between this algorithm 
and Algorithm~\ref{a:disent}.
The first difference is that each cell in the dynamic programming table $f_{c,j}$, $c \in [0, \cU]$, $j \in \{0,1,\ldots,n-1\}$, where $\cU = max(\C) - |P_k|$, is a triple of values $q_{c,j}$, $l_{c,j}$ and $h_{c,j}$,  $f_{c,j} =\{q_{c,j}, l_{c,j}, h_{c,j}\}$. The new parameter $h_{c,j}$ stores the number of neutral variables covered by  $last(S_{c,j})$. The second difference is in the way we deal with neutral variables.
If $x_j \in N_k$  then we have two options now. We can obtain $S_{c,j}$ from $S_{c-1,j-1}$ by increasing
$cst(S_{c-1,j-1})$ by one and increasing the number of covered neutral variables by $last(S_{c,j-1})$ (Figure~\ref{fig:focusw_scheme}(c), the gray arc). Alternatively, we can obtain $S_{c,j}$  from  $S_{c,j-1}$ by interrupting $last(S_{c,j-1})$  (Figure~\ref{fig:focusw_scheme}(c), the black arc).
BC can enforced using two modifications of the corresponding
algorithm for $\FOCUSW$ (a proof is given in Appendix).  

\begin{lemma}\label{l:focuswh_bc}
Consider  $\FOCUSWH(X, \numb,$ $\len, \h, \kv, \C)$.
BC can be enforced in $O(n\max(\C))$ time.
\end{lemma}


$\FOCUSWH$ can be encoded using  the cost-$\REGULAR$ constraint. The 
automaton needse 3 counters to compute $\len, \numb$ and $\h$. Hence, the
time complexity of this encoding is $O(n^4)$.  This automaton
is non-deterministic as on seeing $v \leq \kv$, it 
either covers the variable
or interrupts the last sequence.
Unfortunately the non-deterministic cost-$\REGULAR$
is not implemented in any constraint solver to our
knowledge.
In contrast, our algorithm
takes just $O(n^2)$ time.
$\FOCUSWH$ can also be decomposed using the {\sc Gcc} constraint ~\cite{Regin96}. We define the following variables for all $ i \in [0,max(\numb)-1]$ and $j \in [0,n-1]$: 
$\Start_i$ the start of the $i$th sub-sequence. $D(\Start_i)=\{0,.., n+max(\numb)\}$; 
$\End_i$ the end of the $i$th sub-sequence. $D(\End_i)=\{0,.., n+max(\numb)\}$;
$\Taken_j$ the index of the subsequence in $S_X$ containing $x_j$. $D(\Taken_j)=\{0,.., max(\numb)\}$;
$\Zero_j$ the index of the subsequence in $S_X$ containing $x_j$ s.t. the value of $x_j$ is less than or equal to $k$. $D(\Zero_j)=\{0,.., max(\numb)\}$;
$\LastTaken$ the cardinality of $S_X$. $D(\LastTaken)=\{0,..,max(\numb)\}$; 
$Card$, a vector of $max(\numb)$ variables having $\{0,..,h\}$ as domains. 
{\small
\FOCUSWH($X,\numb, \len, \h, \kv, \C$) $\Leftrightarrow$
\begin{equation}
 \begin{split}
  (x_j \leq \kv) \vee \Zero_j=0; & \qquad  (x_j \leq \kv) \vee \Taken_j>0;\nonumber\\
  (x_j > \kv) \vee (\Taken_j=\Zero_j); & \qquad (\Taken_j \leq \LastTaken); \\
 (\Taken_j \neq i) \vee (j \geq \Start_{i-1});   & \qquad(\Taken_j \neq i) \vee (j\leq \End_{i-1});   \\
   (i>\LastTaken) \vee (\Taken_j = i) \vee & (j<\Start_{i-1})\vee (j>\End_{i-1}); \\
 \forall q \in [1,max(\numb)-1], &  \qquad q \geq \LastTaken \vee \Start_q> \End_{q-1};\\
 \forall q \in [0,max(\numb)-1]   &  \qquad q \geq \LastTaken \vee  \End_q\geq \Start_q; \\
  \forall q \in [0,max(\numb)-1]   &  \qquad  q \geq \LastTaken \vee len > (\End_q-\Start_q); \\
  \end{split}\end{equation}
\noindent
\vspace*{-0.4cm}
  \begin{equation}
    \begin{split}
 \LastTaken \leq \numb;  \qquad  {\sc Gcc} ([\Taken_0,..,\Taken_{n-1}],\{0\},[n- \C]); \\
  {\sc Gcc}([\Zero_0,..,\Zero_{n-1}], \{1,..,max(\numb)\}, Card)\nonumber; \\
  \end{split}
  \end{equation}
}
\vspace{-0.5cm}
\section{Experiments}\label{sec:exps}
We used the Choco-2.1.5 solver on an IntelXeon 2.27GHz for the first benchmarks and IntelXeon 3.20GHz for last ones, both under Linux. We compared the propagators (denoted by F) of $\FOCUSW$ and $\FOCUSWH$ against two decompositions (denoted by D$_1$ and D$_2$), using the same search strategies, on three different benchmarks. The first decomposition, restricted to $\FOCUSW$, is shown in proposition \ref{prop:decomp_focusw}, while the second one is shown in Section \ref{sec:focuswh}. 
In the
tables, we report for each set the total number of solved instances (\#n), then we average both the number of backtracks (\#b) and the resolution time (T) in seconds. 

\noindent $\Box$ \emph{Sports league scheduling (SLS).} We extend a single round-robin problem
with $n=2p$ teams. 
Each week each team plays a game either
at home or away. Each team plays exactly once all the other teams during a season.
We minimize the number of breaks (a break for one team is two consecutive home or two consecutive away games), while fixed weights in $\{0,1\}$ are
assigned to all games: games with weight 1
are important for TV channels. The goal is to group consecutive weeks where
at least one game is important (sum of weights $>0$), to increase the price of TV broadcast packages. Packages are 
limited to 5 weeks and should be as short as possible. Table~\ref{tab:SLS} shows results with 16 and 20 teams, 
on sets of 50 instances with 10 random important games and a limit of 400K backtracks. 
$\max(\numb)=3$ and we search for one solution with $\h \leq 7$ (instances $n$-1),
$h\leq6$ ($n$-2) and $h\leq5$ ($n$-3). 
In our model, inverse-channeling and {\sc AllDifferent} constraints  with the strongest propagation level express  
that each team plays once against each other team. 
We assign first the sum of breaks by team, then the breaks and places using the 
\emph{DomOverWDeg} strategy.
%
{
\begin{table}[t]
\caption{\label{tab:SLS}SLS with $\FOCUSW$ and its decomposition.}
\begin{scriptsize}
\renewcommand{\tabcolsep}{2.6pt}
\begin{center}
\begin{tabular}{|c|c|c|c|c|c|c|c|c|c|cl}
\hline
 & \multicolumn{ 3 }{c|}{16\_1} &  \multicolumn{ 3 }{c|}{16\_2} &  \multicolumn{ 3 }{c|}{16\_3} \\
\hline
\hline
 &\#n & \#b & T & \#n & \#b & T & \#n & \#b & T    \\
\hline
 F &  50 & 0.9K & 0 & {\bf 50} & 4.1K  & 2 & {\bf 47} & 18.1K & 7    \\
 D$_1$ &   50 & 3.4K& 1 & 49 & 8.1K & 3 &  44 & 21.8K& 8   \\
 \hline
\end{tabular}
\begin{tabular}{|c|c|c|c|c|c|c|c|c|c|cl}
\hline
&  \multicolumn{ 3 }{c|}{20\_1} &  \multicolumn{ 3 }{c|}{20\_2} &  \multicolumn{ 3 }{c|}{20\_3} \\
\hline
\hline
 &\#n & \#b & T & \#n & \#b & T & \#n & \#b & T    \\
 \hline
F & {\bf 49} & 11.8K &  7 & {\bf 45} & 24.9K & 14  & {\bf 39} & 36.5K & 23    \\
D$_1$ & 43 &30.8K & 13  & 35 & 27.2K  & 12  &  29  & 29.6K & 17 \\
 \hline
\end{tabular}
\end{center}
\end{scriptsize}
\end{table}
}

%

\noindent $\Box$ \emph{Cumulative Scheduling with Rentals.}
Given a horizon of $n$ days and a set of time intervals $[s_i,e_i]$, $i\in \{1,2,\ldots,p\}$,
a company needs to rent a machine between $l_i$ and $u_i$ times within each time interval $[s_i,e_i]$.
We assume that the cost of the rental period is proportional to its length.
On top of this, each time the machine is rented we pay a fixed cost.
The problem is then
defined as a conjunction of one \FOCUSWH($X,\numb, \len, \h, 0, \C$) with a set of \AMONG\ constraints.
The goal is to build a schedule for rentals that satisfies all demand constraints
and minimizes simultaneously the number of rental periods and their total length.
We build a Pareto frontier over two cost variables, as Figure~\ref{fig:pareto}
shows 
for one of the instances of this problem.
\begin{figure}[!h]
\begin{center}
  \includegraphics[width=0.45\textwidth]{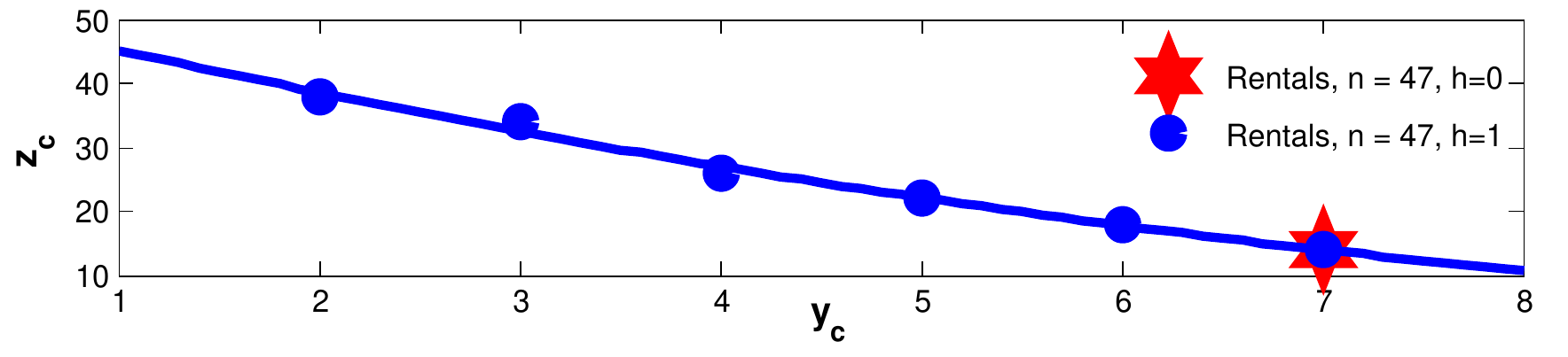}
\end{center}
\vspace{-0,55cm}
\caption{Pareto frontier for Scheduling with Rentals.}
  \label{fig:pareto}
\end{figure}
We generated instances having a fixed length of sub-sequences of size 20 (i.e., $len=20$), 50\% as a probability of posting an \emph{Among} constraint for each $(i,j)$ s.t. $j\geq i+5$ in the sequence. Each set of instances corresponds to a unique sequence size ($\{40,43,45,47,50\}$) and 20 different seeds. We summarize these tests in table \ref{tab:rental}. Results with decomposition are very poor. We therefore consider only the propagator in this case.

%
%

\begin{table}[t]
\caption{\label{tab:rental}Scheduling with Rentals.}
\begin{scriptsize}
\renewcommand{\tabcolsep}{2.7pt}
\begin{center}
\begin{tabular}{|c|c|c|c|c|c|c|c|c|c|c|}
\hline &  & \multicolumn{3}{c|}{40} & \multicolumn{3}{c|}{43} & \multicolumn{3}{c|}{ 45}  \\
\hline
\hline  h  & & \#n & \#b & T  & \#n & \#b & T  & \#n & \#b & T  \\
\hline  0 & F  &    20 & 349K& 55.4   &     20 & 1M &  192.2   &    20  &  1M &233.7  \\
0 & D$_1$  &   20 & 529K &74.7  &   20   & 1M & 251.2   &   20  &  1M & 328.6 \\
\hline 1 & F  &  20  & 827M & 120.4 &   20  &  2M & 420.9  &   19 & 3M & 545.9  \\
\hline 2 & F  &    20  &    826K &      115.7   &   20 & 2M &   427.3  &  19& 3M &   571.3   \\
\hline 
\end{tabular}

\begin{tabular}{|c|c|c|c|c|c|c|c|}
\hline &  & \multicolumn{3}{c|}{47} & \multicolumn{3}{c|}{50} \\
\hline
%
\hline  h  &    & \#n & \#b     & T      & \#n & \#b    & T \\ 
\hline  0 & F  &  {\bf 19} &   1M &  354.5   &    {\bf 18} &  2M &   553.7   \\
0 & D$_1$ &  18 & 2M & 396.8  &   17 &  3M & 660   \\
\hline 1 & F  &    16&  4M & 725.4  &   4 &   6M &   984.5  \\
\hline 2 & F  &    15 &  4M &  763.9 &   4 &  5M & 944.8  \\
\hline
\end{tabular}
\end{center}
\end{scriptsize}
\end{table}

\noindent $\Box$ \emph{Sorting Chords.}
We need to sort $n$ distinct chords. Each chord is a set of at most $p$ notes played simultaneously.
The goal  is to find an ordering that minimizes the number of notes changing between two consecutive chords.
The full description and a CP model is in~\cite{petit12}. The main
 difference here is that 
we build a Pareto frontier over two cost variables.
We generated 4 sets of instances distinguished by the numbers of chords ($\{14,16,18,20\}$). We fixed the length of the subsequences and the maximum notes for all the sets then change the seed for each instance.





{
\begin{table}[t]
\caption{\label{tab:SC}Sorting Chords}
\centering
\begin{scriptsize}
\renewcommand{\tabcolsep}{2.7pt}
\begin{tabular}{|c|c|c|c|c|c|c|c|c|c|c|c|c|c|}
\hline   &  &  \multicolumn{ 3 }{c|}{ 14} & \multicolumn{ 3 }{c|}{ 16} & \multicolumn{ 3 }{c|}{ 18} & \multicolumn{ 3 }{c|}{20} \\
\hline\hline  h&  &  \#n & \#b & T & \#n & \#b & T & \#n & \#b & T & \#n & \#b & T \\
\hline 0 & F    &    30& 70K & 2.8  &   30 &  865K  &  14.6   &    28 & 10M &    182.9   &   {\bf 16} & 14M &  270.4  \\
  0 & D$_1$  &    30 & 94K &   3.2  &    30 &   2M &   41 &    28 & 12M & 206.9  &    13 & 10M & 206.8  \\
 0 & D$_2$  &  30 &  848K & 34.9  &    24   & 3M &  122.3  &   13 &   8M &  285.6  &    7 &902K &  38.7  \\
\hline 1 & F  &  30 &   97K & 3.5   &    {\bf 30} & 1K &   27.2 &    {\bf 28} &  12K &  214.2  &     {\bf 14} &  13M & 288.2 \\
 1 & D$_2$  &    30 & 851M &  41.5   &   23 & 2M  &       116.3 &    11 &  5M  &    209.9  &    7&  868K  &   41.5  \\
\hline 2 & F  &     30 &   97K &  3.4  &   {\bf  30} & 1M &    25.9  &  {\bf 28} &13M &  217.4   &      {\bf 13} &  12M&  245.5   \\
 2 & D$_2$  &    30  &  844K  & 40.9  &   24 &3M  &    145.1  &     12  &  6K & 251.6 &      7 & 867K  &       42.8 \\
\hline
\end{tabular}
\end{scriptsize}
\end{table}
}

Tables \ref{tab:SLS}, \ref{tab:rental} and \ref{tab:SC} show that best results 
were obtained with our propagators (number of solved instances, average backtracks and CPU time over all the solved instances\footnote{While the technique that solves the largest number of instances (and thus some  harder ones) should be penalized.}).
Figure~\ref{fig:pareto} confirms the gain of flexibility illustrated by Figure~\ref{fig:SF} in Section~\ref{sec:focush}: allowing $\h = 1$ variable with a low cost value into each sequence leads to
new solutions, with significantly lower values for the target variable $\numb$.

\section{Conclusion}\label{sec:concl}
We have presented flexible tools for capturing the concept of concentrating costs. Our contribution highlights the expressive power of constraint programming,
in comparison with other paradigms where such a concept would be very difficult to represent. Our experiments have demonstrated the effectiveness of the proposed new filtering algorithms. 

\pagebreak
\bibliographystyle{named}
\bibliography{lit}

~\\~\\~\\~\\~\\~\\~\\~\\~\\~\\~\\~\\~\\~\\~\\~\\
\section*{Appendix}\label{sec:app}
~\\
\begin{appendix}
Proofs of Lemmas~\ref{lem:initialization} to \ref{lem:pgt} ommit the obvious cases where quantities take the default value $n+1$.
\section{Proof of Lemma~\ref{lem:initialization}}\label{app:lem:initialization}
From item 4 of Definition~\ref{def:SF}, a sequence in $S_X$ cannot start with a value $v \leq k$.
Thus, $\underline{p_{S}}(x_{0},v_\leq)=n+1$ and $\underline{\mathit{card}}(x_{0})$ $=$ $0$.
If $x_0$ can take a value $v > k$ then by Definition~\ref{def:recSF}, $\underline{p}(x_{0},v_>)$ $=$ $1$ and $\underline{\mathit{plen}}(x_{0})$ $=$ $1$. \qed
\section{Proof of Lemma~\ref{lem:pleq}}\label{app:lem:pleq}
 If $\min(x_l) \leq k$ then
$\underline{p_{S}}(x_{l-1},v_\leq)$ must not be considered: it would
imply that a sequence in $S_X$ ends by a value $v \leq k$ for $x_{l-1}$.
From Property~\ref{prop:fc}, the focus cardinality of the previous sequence is
$\min(\underline{p}(x_{l-1},v_\leq), \underline{p}(x_{l-1},v_>))$. \qed
\section{Proof of Lemma~\ref{lem:pleqs}}\label{app:lem:pleqs}
 If $\min(x_i) \leq k$ we have three cases to consider. (1) If  either $\underline{\mathit{plen}}(x_{i-1})=0$ or $\underline{\mathit{plen}}(x_{i-1})=\len$ then
from item 3 of Definition~\ref{def:SF} a sequence in $S_X$ cannot start with a value $v_i \leq k$: $\underline{p_S}(x_{i},v_\leq)$ $=$ $n+1$. (2) If $\underline{\mathit{plen}}(x_{i-1})=\len-1$ then from Defiinition~\ref{def:SF} the current variable $x_i$ cannot end the sequence
with a value $v_i \leq k$. (3) Otherwise, from item 3 of Definition~\ref{def:SF}, $\underline{p}(x_{i-1},v_\leq)$ is not considered. Thus, from Property~\ref{prop:fc}, $\underline{p_S}(x_{i},v_\leq)$ $=$ $\min(\underline{p_S}(x_{i-1},v_\leq), \underline{p}(x_{i-1},v_>))$.
\qed
\section{Proof of Lemma~\ref{lem:pgt}}\label{app:lem:pgt}
If  $\underline{\mathit{plen}}(x_{l-1})$ $\in$ $\{0, \len\}$ a new sequence has to be considered: $\underline{p_S}(x_{l-1},v_\leq)$ must not be considered, from item 3 of Definition~\ref{def:SF}.
Thus, $\underline{p}(x_{l},v_>)$ $=$ $\min(\underline{p}(x_{l-1},v_>)+1, \underline{p}(x_{l-1},v_\leq) +1)$. Otherwise, either a new sequence has to be considered ($ \underline{p}(x_{l-1},v_\leq) +1$)
or the value is equal to the focus cardinality of the previous sequence ending in $x_{l-1}$. \qed
\section{Proof of Lemma~\ref{l:dp_prop_mon}}\label{sec:app:l:dp_prop_mon}
\sloppy
First, we prove two technical results.
\begin{lemma}\label{l:dp_prop_conseq}
Consider  $\FOCUSW([x_0,\ldots,x_{n-1}],$ $ y_c, len,k, \C)$.
Let $f$ be dynamic programming table returned by Algorithm~\ref{a:disent}.
Then the non-dummy values of $f_{c,j}$ are consecutive  in each column,
so that there do not exist $c,c', c''$, $0 \leq c < c' < c'' \leq \cU$, such that
$f_{c',j}$ is dummy and $f_{c,j},f_{c'',j}$ are non-dummy.
\end{lemma}
\begin{proof}
We prove by induction on the length of the sequence.
The base case is trivial as
$f_{0,-1} = \{0,0\}$ and $f_{c,-1} = \{\infty,\infty\}$, $c \in [-1] \cup [1,\cU]$.
Suppose the statement holds for $j-1$ variables.

Suppose there exist $c, c', c''$, $0 \leq c < c' < c'' \leq \cU$, such that
$f_{c',j}$ is dummy and $f_{c,j},f_{c'',j}$ are non-dummy.

\textbf{Case 1.} Consider the case $x_j \in P_k$.
By Algorithm~\ref{a:disent},  lines~\ref{a:cond_1} and~\ref{a:cond_2},
$q_{c,j} \in [q_{c,j-1},q_{c,j-1}+1]$,
$q_{c',j} \in [q_{c',j-1},q_{c',j-1}+1]$ and
$q_{c'',j} \in [q_{c'',j-1},q_{c'',j-1}+1]$.
As $f_{c',j}$ is dummy and $f_{c,j},f_{c'',j}$ are non-dummy,
$f_{c',j-1}$ must be dummy and $f_{c,j-1},f_{c'',j-1}$ must be non-dummy.
This violates  induction hypothesis.

\textbf{Case 2.} Consider the case $x_j \in U_k$.
By Algorithm~\ref{a:disent},  lines~\ref{a:cond_3} and~\ref{a:cond_4},
$q_{c,j} = min(q_{c-1,j-1},q_{c,j-1})$,
$q_{c',j} = min(q_{c'-1,j-1},q_{c',j-1})$ and
$q_{c'',j} = min(q_{c''-1,j-1},q_{c'',j-1})$.
As $f_{c',j}$ is dummy, then both
$f_{c'-1,j-1}$ and $f_{c',j-1}$ must be dummy.
As $f_{c,j}$ is non-dummy, then one of
$f_{c-1,j-1}$ and $f_{c,j-1}$ is non-dummy.
As $f_{c'',j}$ is non-dummy, then one of
$f_{c''-1,j-1}$ and $f_{c'',j-1}$ is non-dummy.
We know that $c-1<c \leq c'-1 < c' \leq c''-1 < c''$
or $c  < c'  < c''$.
This leads to violation of  induction hypothesis.

\textbf{Case 3.} Consider the case $x_j \in N_k$.
By Algorithm~\ref{a:disent},  line~\ref{a:cond_5},
$q_{c,j} = q_{c,j-1}$,
$q_{c',j} = q_{c',j-1}$ and
$q_{c'',j} = q_{c'',j-1}$.
Hence, $f_{c',j-1}$ is dummy and $f_{c,j-1},f_{c'',j-1}$ are non-dummy.
This leads to violation of  induction hypothesis.
\end{proof}

\begin{proposition}\label{l:dp_first_row}
Consider   $\FOCUSW([x_0,\ldots,x_{n-1}],$ $ y_c, len,k, \C)$.
Let $f$ be dynamic programming table returned by Algorithm~\ref{a:disent}.
The elements of the first row are non-dummy:
$f_{0,j}$, $j=-1,\ldots,n$ are non-dummy.
\end{proposition}
\begin{proof}
We prove by induction on the length of the sequence.
The base case is trivial as $f_{0,-1} = \{0,0\}$.
Suppose the statement holds for $j-1$ variables.

\textbf{Case 1.} Consider the case $x_j \in P_k$.
As $f_{0,j-1}$ is non-dummy then by Algorithm~\ref{a:disent},  lines~\ref{a:cond_1}--~\ref{a:cond_2},
$f_{0,j}$ is non-dummy.

\textbf{Case 2.} Consider the case $x_j \in U_k$.
Consider the condition
$(l_{-1,j-1} \in [1,len) \wedge q_{-1,j-1} = q_{0,j-1}) \vee (q_{0,j-1} = \infty)$
at line~\ref{a:cond_U_k}.
By the induction hypothesis, $q_{0,j-1} \neq \infty$. By the initialization
procedure of the dummy row, $q_{-1,j-1} = \infty$. Hence, this condition
does not hold and, by line~\ref{a:cond_4}, $f_{0,j}$ is non-dummy.

\textbf{Case 3.} Consider the case $x_j \in N_k$.
As $f_{0,j-1}$ is non-dummy then by Algorithm~\ref{a:disent},  lines~\ref{a:cond_5},
$f_{0,j}$ is non-dummy.
\end{proof}

We can now prove~Lemma~\ref{l:dp_prop_mon}.
\begin{proof}
By transitivity and consecutivity of non-dummy values (Lemma~\ref{l:dp_prop_conseq})
and the result that all elements in the $0$th row are non-dummy (Proposition~\ref{l:dp_first_row}),
it is sufficient to consider the case $c'=c+1$.

We prove by induction on the length of the sequence.
The base case is trivial as
$f_{0,-1} = \{0,0\}$ and $f_{c,0}$ are dummy, $c \in [0,\cU]$.
Suppose the statement holds for $j-1$ variables.

Consider the variable $x_j$. Suppose, by contradiction, that $f_{c,j} < f_{c+1,j}$.
Then either  $q_{c,j} < q_{c+1,j}$ or $q_{c,j} = q_{c+1,j}, l_{c,j} < l_{c+1,j}$.
By induction hypothesis, we know that $f_{c,j-1} \geq f_{c+1,j-1}$, hence,
either  $q_{c,j-1} > q_{c+1,j-1}$ or $q_{c,j-1} = q_{c+1,j-1}, l_{c,j-1} \geq l_{c+1,j-1}$.

We consider three cases depending on whether $x_j$ is a penalizing variable,
an undetermined variable or a neutral variable.

\textbf{Case 1.} Consider the case $x_j \in P_k$.
If $q_{c,j-1} = \infty$ then $q_{c+1,j-1} = \infty$ by the induction hypothesis.
Hence, by Algorithm~\ref{a:disent},  line~\ref{a:cond_1},
$f_{c,j}$ and $f_{c+1,j}$ are dummy and equal.
Suppose $q_{c,j-1} \neq \infty$. Then  we consider four cases based on relative values of
$q_{c,j'},q_{c+1,j'},l_{c,j'},l_{c+1,j'}$, $j' \in\{j-1,j\}$.
\begin{itemize}
  \item
\textbf{Case 1a.} Suppose $q_{c,j} < q_{c+1,j}$ and $q_{c,j-1} > q_{c+1,j-1}$.
By Algorithm~\ref{a:disent},  lines~\ref{a:cond_1} and~\ref{a:cond_2},  $q_{c,j} \geq q_{c,j-1}$
and $q_{c+1,j} \leq q_{c+1,j-1}+1$. Hence,
$q_{c,j} < q_{c+1,j}$ implies $\mathbf{q_{c+1,j-1}} < q_{c,j} <  \mathbf{q_{c+1,j-1}+1}$.
We derive a contradiction.
  \item \textbf{Case 1b.} Suppose $q_{c,j} < q_{c+1,j}$ and $q_{c,j-1} = q_{c+1,j-1}, l_{c,j-1} \geq l_{c+1,j-1}$.
By Algorithm~\ref{a:disent},  lines~\ref{a:cond_1} and~\ref{a:cond_2},  $q_{c,j} \geq q_{c,j-1}$
and $q_{c+1,j} \leq q_{c+1,j-1}+1$.  Hence,
$q_{c,j} < q_{c+1,j}$ implies $\mathbf{q_{c+1,j-1}} = q_{c,j-1} \leq q_{c,j} <  q_{c+1,j} \leq \mathbf{q_{c+1,j-1}+1}$.
Hence, $q_{c+1,j-1} = q_{c,j-1} = q_{c,j}$  and $q_{c+1,j} ={q_{c+1,j-1}+1}$.
As $q_{c,j-1} = q_{c,j}$ then $l_{c,j-1} \in [1,len)$ by Algorithm~\ref{a:disent} line~\ref{a:cond_1}.
As $q_{c+1,j} = q_{c+1,j-1}+1$ then $l_{c+1,j-1} \in \{len,\infty\}$ by Algorithm~\ref{a:disent} line~\ref{a:cond_2}.
This leads to a contradiction as $l_{c,j-1} \geq l_{c+1,j-1}$.

 \item \textbf{Case 1c.} Suppose $q_{c,j} = q_{c+1,j}, l_{c,j} < l_{c+1,j}$ and $q_{c,j-1} > q_{c+1,j-1}$.
Symmetric to Case 1b.
  \item
\textbf{Case 1d.}
Suppose  $q_{c,j} = q_{c+1,j}, l_{c,j} < l_{c+1,j}$ and $q_{c,j-1} = q_{c+1,j-1}, l_{c,j-1} \geq l_{c+1,j-1}$.
By Algorithm~\ref{a:disent},  lines~\ref{a:cond_1} and~\ref{a:cond_2},  $q_{c,j} \geq q_{c,j-1}$
and $q_{c+1,j} \leq q_{c+1,j-1}+1$.  Hence,
$q_{c,j} = q_{c+1,j}$ implies $\mathbf{q_{c+1,j-1}} = q_{c,j-1} \leq q_{c,j} =  q_{c+1,j} \leq \mathbf{q_{c+1,j-1}+1}$.
Therefore, either  $q_{c,j} = q_{c,j-1}\ \wedge\ q_{c+1,j} = q_{c+1,j-1}$ or
$q_{c,j} = q_{c,j-1}+1\ \wedge\ q_{c+1,j} = q_{c+1,j-1}+1$.

If $q_{c,j} = q_{c,j-1}$ and $q_{c+1,j} = q_{c+1,j-1}$ then
$l_{c,j-1} \in [1,len)$ and $l_{c+1,j-1} \in [1,len)$ by Algorithm~\ref{a:disent} line~\ref{a:cond_1}.
Hence, $l_{c,j} = l_{c,j-1}+1$ and  $l_{c+1,j} = l_{c+1,j-1}+1$.
As $l_{c,j-1} \geq l_{c+1,j-1}$, then  $l_{c,j} \geq l_{c+1,j}$.
This leads to a contradiction with the assumption $l_{c,j} < l_{c+1,j}$.

If $q_{c,j} = q_{c,j-1}+1\ \wedge\ q_{c+1,j} = q_{c+1,j-1}+1$  then
$l_{c,j-1} \in \{len,\infty\}$ and $l_{c+1,j-1}  \in \{len,\infty\}$
by Algorithm~\ref{a:disent} line~\ref{a:cond_2}.
Hence, $l_{c,j} = 1$ and  $l_{c+1,j} = 1$. This leads to a contradiction
with  the assumption $l_{c,j} < l_{c+1,j}$.
\end{itemize}

\textbf{Case 2.} Consider the case $x_j \in U_k$.
If $q_{c,j-1} = \infty$ then $q_{c+1,j-1} = \infty$ by the induction hypothesis.
Hence, by Algorithm~\ref{a:disent},  line~\ref{a:cond_3},
$f_{c,j}$ and $f_{c+1,j}$ are dummy and equal.

Suppose $q_{c,j-1} \neq \infty$. Then we consider four cases based on relative values of
$q_{c,j'},q_{c+1,j'},l_{c,j'},l_{c+1,j'}$, $j' \in\{j-1,j\}$.

\begin{itemize}
  \item
\textbf{Case 2a} Suppose $q_{c,j} < q_{c+1,j}$ and $q_{c,j-1} > q_{c+1,j-1}$.
By Algorithm~\ref{a:disent},  lines~\ref{a:cond_3} and~\ref{a:cond_4},  we know that
$q_{c+1,j-1} \leq q_{c+1,j} \leq q_{c,j-1}$
and
$q_{c,j-1} \leq q_{c,j} \leq q_{c-1,j-1}$.
By induction hypothesis,  $q_{c+1,j-1} \leq q_{c,j-1} \leq q_{c-1,j-1}$.
Hence, if $ q_{c,j} \leq q_{c+1,j}$ then
$\mathbf{q_{c,j-1}} \leq q_{c,j} \leq q_{c+1,j} \leq  \mathbf{q_{c,j-1}}$. Therefore,
if $q_{c,j} < q_{c+1,j}$ then we derive a contradiction.

\item \textbf{Case 2b.} Identical to Case 2b.


  \item \textbf{Case 2c.} Suppose $q_{c,j} = q_{c+1,j}, l_{c,j} > l_{c+1,j}$ and $q_{c,j-1} > q_{c+1,j-1}$.
As  $q_{c,j-1} \neq q_{c+1,j-1}$ then $q_{c+1,j-1} = q_{c+1,j}$ ( line~\ref{a:cond_3}).
We also know $\mathbf{q_{c,j-1}} \leq q_{c,j} \leq q_{c+1,j} \leq \mathbf{ q_{c,j-1}}$ from Case 1a.
Putting everything together, we get $\mathbf{q_{c,j-1}} \leq q_{c,j} \leq q_{c+1,j-1} < \mathbf{q_{c,j-1}}$.
This leads to a contradiction.
  \item
\textbf{Case 2d.} Suppose  $q_{c,j} = q_{c+1,j}, l_{c,j} < l_{c+1,j}$ and $q_{c,j-1} = q_{c+1,j-1}, l_{c,j-1} \geq l_{c+1,j-1}$.
As we know from Case 1a
$q_{c+1,j-1} \leq q_{c+1,j} \leq q_{c,j-1}$, $q_{c,j-1} \leq q_{c,j} \leq q_{c-1,j-1}$ and
$\mathbf{q_{c,j-1}} \leq q_{c,j} \leq q_{c+1,j} \leq  \mathbf{q_{c,j-1}}$.
Hence, $q_{c+1,j-1}  = q_{c+1,j} = q_{c,j-1} = q_{c,j}$.

Consider two subcases.  Suppose $q_{c,j-1} < q_{c-1,j-1}$.
Then $l_{c,j} = \infty$  (line~\ref{a:cond_4}). Hence,
our assumption $l_{c,j} < l_{c+1,j}$ is false.

Suppose $q_{c,j-1} = q_{c-1,j-1}$. If $l_{c-1,j-1} = len$
then $l_{c,j} = \infty$  (line~\ref{a:cond_4}).
Hence, our assumption $l_{c,j} < l_{c+1,j}$ is false.
Therefore, $l_{c-1,j-1} \in [1,len)$ and
$l_{c,j-1} = l_{c-1,j-1}+ 1$.
By induction hypothesis as $q_{c+1,j-1} = q_{c,j-1} = q_{c-1,j-1}$
then $l_{c+1,j-1} \leq l_{c,j-1} \leq l_{c-1,j-1}$.
Hence,  $l_{c,j-1} \in [1,l_{c-1,j-1}] \subseteq [1,len)$.
Therefore, $l_{c+1,j} = l_{c,j-1} + 1 \leq l_{c-1,j-1}+ 1=l_{c,j-1}$.
This contradicts our assumption  $l_{c,j} < l_{c+1,j}$.

\textbf{Case 3.} Consider the case $x_j \in N_k$.
This case follows immediately from Algorithm~\ref{a:disent},  line~\ref{a:cond_5},
and the induction hypothesis.
\end{itemize}
\end{proof}

\section{Proof of Lemma~\ref{l:disent}}\label{sec:app:l:disent}

\begin{proof}[Proof of correctness.]
We prove by induction on the length of the sequence.
Given $f_{c,j}$ we can reconstruct a corresponding set
of sequences $S_{c,j}$ by traversing the table backward.

The base case is trivial as $x_1 \in P_k$, $f_{0,0} = \{1,1\}$ and $f_{c,0} = \{\infty,\infty\}$.
Suppose the statement holds for $j-1$ variables.

\textbf{Case 1.} Consider the case $x_j \in P_k$.  Note, that the cost can not be increased
on seeing $x_j \in P_k$ as cost only depends on  covered undetermined variables.
By the induction hypothesis, $S_{c,j-1}$ satisfies conditions~\ref{cond:cover}--\ref{cond:last}.
The only way to obtain $S_{c,j}$ from $S_{c',j-1}$, $c'\in[0,\cU]$,
is to extend  $last(S_{c,j-1})$ to cover $x_j$ or start a new sequence
if $|last(S_{c,j-1})| = len$. If $S_{c,j-1}$ does not exist then
$S_{c,j}$ does not exist.
The algorithm performs this extension (lines~\ref{a:cond_1} and~\ref{a:cond_2}).
Hence, $S_{c,j}$ satisfies conditions~\ref{cond:cover}--\ref{cond:last}.

\textbf{Case 2.} Consider the case $x_j \in U_k$.
In this case, there exist two options to obtain $S_{c,j}$
from from $S_{c',j-1}$, $c'\in[0,\cU]$.

The first option is to cover $x_j$. Hence, we need to extend $last(S_{c-1,j-1})$.
Note that we should not start a new sequence if $last(S_{c-1,j-1}) = len$
as it is never optimal to start a sequence on seeing a neutral variable.

The second option is not to cover $x_j$. Hence, we need to interrupt $last(S_{c,j-1})$.

By Lemma~\ref{l:dp_prop_mon} we know that $f_{c,j-1} \leq f_{c-1,j-1}$, $0 < c \leq C$.
By the induction hypothesis, $S_{c,j-1}$ and $S_{c-1,j-1} $ satisfy conditions~\ref{cond:cover}--\ref{cond:last}.
Hence, $S_{c,j-1} \leq S_{c-1,j-1}$.

Consider two cases. Suppose $|S_{c,j-1}| < |S_{c-1,j-1}|$.
In this case, it is optimal to interrupt $last(S_{c,j-1})$.

Suppose $|S_{c,j-1}| = |S_{c-1,j-1}|$ and $|last(S_{c,j-1})| \leq |last(S_{c-1,j-1})|$.
If $|last(S_{c-1,j-1})| < len$ then  it is optimal to extend $last(S_{c-1,j-1})$.
If  $|last(S_{c-1,j-1})| = len$ then  it is optimal to interrupt $last(S_{c,j-1})$,
otherwise we would have to start a new sequence to cover an undetermined variable $x_j$, which is never optimal.
If $S_{c,j-1}$ and $S_{c-1,j-1}$ do not exist then
$S_{c,j}$ does not exist.
If $S_{c,j-1}$ does not exist then case analysis is similar
to the analysis above.

This case-based analysis is exactly what Algorithm~\ref{a:disent} does in lines~\ref{a:cond_3}
and~\ref{a:cond_4}.
Hence, $S_{c,j}$ satisfies conditions~\ref{cond:cover}--\ref{cond:last}.

\textbf{Case 3.} Consider the case $x_j \in N_k$.  Note that the cost can not be increased
on seeing $x_j \in N_k$ as cost only depends on  covered undetermined variables.
By the induction hypothesis, $S_{c,j-1}$ satisfies conditions~\ref{cond:cover}--\ref{cond:last}.
The only way to obtain $S_{c,j}$ from $S_{c',j-1}$, $c'\in[0,\cU]$,
is to interrupt $last(S_{c,j-1})$. If $S_{c,j-1}$ does not exist then
$S_{c,j}$ does not exist.
The algorithm performs this extension in line~\ref{a:cond_5}.
Hence, $S_{c,j}$ satisfies conditions~\ref{cond:cover}--\ref{cond:last}.
\end{proof}
\begin{proof}[Proof of complexity.]
The time complexity of the algorithm is $O(n\max(\C)) = O(n^2)$ as we have $O(n\max(\C))$ elements in the table
and we only need to inspect a constant number of elements to compute $f(c,j)$.
\end{proof}

\section{Proof of Lemma~\ref{l:focuswh_bc}}\label{sec:app:focuswh_bc}
First, we present explicitly the algorithm for detecting disentailment.
\begin{algorithm}
\label{a:disent_focuswh}
{\scriptsize
  \caption{$\FOCUSWH$($x_0, \ldots, x_{n-1}$)}
  \For{$c \in -1..\cU$}{ \label{a:for_loop_dummy_focuswh}
      \For{$j \in -1..n-1$}{
         $f_{c,j} \gets \{\infty,\infty,\infty\}$\;
      }
  }
  $f_{0,-1} \gets \{0,0,0\}$
  \;
  \label{a:first_col_focuswh}
  \For{$j \in 0..n-1$}{ \label{a:for_loop_focuswh}
      \For{$c \in 0..j$}{
        \If(\tcc*[f]{penalizing}){$ x_j \in P_k$}{
            \uIf
            {$(l_{c,j-1} \in [1,len)) \vee (q_{c,j-1} = \infty)$}{\label{a:cond_P_k_focuswh}
                 $f_{c,j} \gets  \{q_{c,j-1}, l_{c,j-1} + 1, h_{c,j-1}\}$\; \label{a:cond_1_focuswh}
            }
              \Else
              {
                 $f_{c,j} \gets   \{q_{c,j-1}+1, 1, 0\}$\; \label{a:cond_2_focuswh}
              }
        }				
        \If(\tcc*[f]{undetermined }){$ x_j \in U_k$} {
            \eIf{$(l_{c-1,j-1} \in [1,len) \wedge q_{c-1,j-1} = q_{c,j-1}) \vee (q_{c,j-1} = \infty)$}{\label{a:cond_U_k_focuswh}
                  $f_{c,j} \gets  \{q_{c-1,j-1}, l_{c-1,j-1} + 1, h_{c-1,j-1}\}$
                  \label{a:cond_3_focuswh}
            }{
                   $f_{c,j} \gets   \{q_{c,j-1}, \infty, \infty\} $   
                   \label{a:cond_4_focuswh}
            }
        }

        \If(\tcc*[f]{ neutral }){$ x_j \in N_k$}{
             \eIf{$(l_{c-1,j-1} \in [1,len) \wedge  h_{c-1,j-1} \in [1,h) \wedge q_{c-1,j-1} = q_{c,j-1}) \vee (q_{c,j-1} = \infty)$}{\label{a:cond_U_k_focuswh}
                  $f_{c,j} \gets  \{q_{c-1,j-1}, l_{c-1,j-1} + 1, h_{c-1,j-1}+1\}$
                  \label{a:cond_5_focuswh}
            }{
                   $f_{c,j} \gets   \{q_{c,j-1}, \infty, \infty\} $   
                   \label{a:cond_6_focuswh}
            }
        }
}
  }
 return $f$\;
 }
\end{algorithm}

The algorithm is based on a dynamic program. Each cell in the dynamic programming table $f_{c,j}$, $c \in [0, \cU]$, $j \in \{0,1,\ldots,n-1\}$, where $\cU = max(\C) - |P_k|$, is a triple of values $q_{c,j}$, $l_{c,j}$ and $h_{c,j}$,  $f_{c,j} =\{q_{c,j}, l_{c,j}, h_{c,j}\}$, stores information about $S_{c,j}$. The new parameter $h_{c,j}$ stores the number of neutral variables covered by $last(S_{c,j})$. Algorithm~\ref{a:disent_focuswh} shows pseudocode of the algorithm.

The intuition behind the algorithm is as follows. We emphasize again that by cost of a cover we mean the number of covered undetermined and neutral variables.
\begin{itemize}
  \item If $x_j \in P_k$ then we do \textbf{not} increase the cost of  $S_{c,j}$ compared to $S_{c,j-1}$ as the cost only depends on $x_j \in U_k$. Hence, the best move for us is to extend $last(S_{c,j-1})$ or  we start a new sequence if it is possible. This is encoded in lines~\ref{a:cond_1_focuswh} and~\ref{a:cond_2_focuswh} of the algorithm.
  \item If $x_j \in U_k$ then we have two options. We can obtain $S_{c,j}$ from $S_{c-1,j-1}$ by increasing
$cst(S_{c-1,j-1})$  by one. This means that $x_i$ will be covered by $last(S_{c,j})$. Note that
this does not increase the number of covered neutral variables by $last(S_{c,j})$ as we can always set $x_j = v$, $v > k$.  Alternatively, from  $S_{c,j-1}$ by interrupting $last(S_{c,j-1})$ if necessary. This is encoded in lines~\ref{a:cond_3_focuswh} and~\ref{a:cond_4_focuswh} of the algorithm.
  \item If $x_j \in N_k$  then we have two options. We can obtain $S_{c,j}$ from $S_{c-1,j-1}$ by increasing
$cst(S_{c-1,j-1})$  by one and increasing the number of covered neutral variables by $last(S_{c,j-1})$.
Alternatively, from  $S_{c,j-1}$ by interrupting $last(S_{c,j-1})$  (lines~\ref{a:cond_5_focuswh}--~\ref{a:cond_6_focuswh}).
\end{itemize}
The proof of correctness mimics the corresponding proof for the $\FOCUSW$ constraint.
We can now prove Lemma~\ref{l:focuswh_bc}. 
%
\begin{proof}
The main idea is identical to the proof of the $\FOCUSW$ constraint.
We only highlight the differences between
the $\FOCUSW$ constraint and the $\FOCUSWH$ constraint.

Consider  a variable-value pair $x_i=v$, $v > k$.
The only difference is in the fourth option.
We denote $h(s_{i,j})$  the number of neutral variables covered by $s_{i,j}$.
Similarly, $h(S)= \sum_{s_{i,j \in S}} h(s_{i,j})$.
\begin{itemize}
\item  The fourth and the cheapest option is to glue
$last(S_{c_1,i-1})$, $x_v$ and $last(S_{c_2,n-i-2})$ to a single sequence if
$|last(S_{c_1,i-1})| +|last(S_{c_2,n-i-2})| < len$
and $h(last(S_{c_1,i-1})) +h(last(S_{c_2,n-i-2})) \leq h$.
Hence, $S_{c_1,i-1}' = S_{c_1,i-1} \setminus last(S_{c_1,i-1})$,
$S_{c_2,n-i-2}' = S_{c_2,n-i-2} \setminus last(S_{c_2,n-i-2})$
and $s' $ is a concatenation of $last(S_{c_1,i-1}), x=v$ and $last(S_{c_2,n-i-2})]$.
Then the union $S = S_{c_1,i-1}' \cup S_{c_2,n-i-2}' \cup \{s'\}$
forms a cover: $cst(S) = c_1+c_2+1$, $|S| = |S_{c_1,i-1}| +  |S_{c_2,n-i-2}| - 1$
and $h(S) = h(last(S_{c_1,i-1})) +h(last(S_{c_2,n-i-2}))$.
\end{itemize}
The rest of the proof is analogous to $\FOCUSW$.

Consider  a variable-value pair $x_i=v$, $v \leq k$.
The main difference is that we have the second option to build a support.
Namely, we glue $S_{c_1,i-1}$, $x_i$ and $ S_{c_2,n-i-2}$.
Hence, if $ c_1+c_2+1 \leq \cU$, $|last(S_{c_1,i-1})| +|last(S_{c_2,n-i-2})| < len$
and $h(last(S_{c_1,i-1})) +h(last(S_{c_2,n-i-2})) < h$ then
we can build a support for $x_i=v$.
The rest of the proof is analogous to $\FOCUSW$. 
\end{proof}
\end{appendix}
\end{document}